\documentclass[pdflatex,sn-nature]{sn-jnl}

\usepackage{graphicx}
\usepackage{amsmath,amssymb}
\usepackage{booktabs}
\usepackage[table]{xcolor}
\usepackage{textcomp}
\usepackage{bm}

\definecolor{famlearn}{HTML}{0072B2}
\definecolor{famfuse}{HTML}{A3487F}
\definecolor{fampose}{HTML}{B07800}
\definecolor{famref}{HTML}{7F7F7F}
\definecolor{bestcell}{HTML}{BFE6D8}
\definecolor{secondcell}{HTML}{E2F3EC}
\definecolor{oursrow}{HTML}{F1FAF6}
\newcommand{\best}{\cellcolor{bestcell}}
\newcommand{\second}{\cellcolor{secondcell}}

\makeatletter
\def\toclevel@paragraph{3}
\def\toclevel@subparagraph{2}
\makeatother
\begin{document}

\title[Residual pose posteriors for simulated grasp planning]{Pose accuracy is not grasp reliability: residual pose posteriors improve simulated robotic grasp planning}

\author[1]{\fnm{Md Selim} \sur{Sarowar}}
\author*[1]{\fnm{Sungho} \sur{Kim}}\email{sunghokim@yu.ac.kr}

\affil[1]{\orgdiv{Advanced Visual Intelligence Lab, Department of Electronic Engineering}, \orgname{Yeungnam University}, \orgaddress{\country{Republic of Korea}}}

\abstract{Model-based robotic grasping estimates an object's six-dimensional (6D) pose and plans on its mesh as if that pose were exact. We show that this assumption reduces simulated grasp success even when the pose passes the standard accuracy test. We built LM-Grasp, a benchmark of 15{,}800 real LINEMOD RGB-D images with grasp outcomes simulated on scanned meshes. FoundationPose meets the ADD(-S) criterion (mean model-point error below 10\% of the object diameter) on 0.997 of test scenes, yet planning on its estimate succeeds in 0.791 of scenes, compared with 0.856 using the annotated pose. Because the estimator reports no uncertainty, we build a pose posterior from its own errors on annotated training frames, resampled from frames with similar observable fit, and score each grasp by its mean success across sampled poses and four equally weighted physics scenarios. Without retraining the estimator, posterior planning yields 0.859 (difference from the annotated-pose result +0.003; 95\% scene-wise paired-bootstrap interval [$-$0.003, +0.009]). This interval is compatible with no difference and a small shortfall; we did not conduct an equivalence test. For pose errors of 5 to 10\% of the diameter, simulated success rises from 0.796 to 0.984. Among test scenes assigned singleton-success conformal prediction sets, observed held-out precision is 0.972, compared with 0.904 for the point-estimate planner; these are empirical precision estimates, not conditional guarantees. On occluded LM-O scenes never used for tuning, posterior planning closes 66\% of the gap to the annotated pose.}


\maketitle

\section{Introduction}\label{sec:intro}

A robot that grasps a known object usually estimates the object's six-dimensional (6D) pose and plans the grasp on the object's model~\cite{xiang2018posecnn,tremblay2018dope,deng2020selfsup}. Pose estimation has become very accurate: on LINEMOD~\cite{hinterstoisser2012}, FoundationPose~\cite{wen2024foundationpose} keeps the average model-point distance (ADD, or ADD-S for symmetric objects~\cite{hodan2016eval}) below a tenth of the object diameter on 0.997 of the test images used here. That tolerance is set for pose benchmarks~\cite{hodan2018bop}, not for grasping: it is several millimetres on a small object and centimetres on a large one, while a parallel-jaw grasp near an edge can fail on a few millimetres of misplaced contact. The planner cannot know this, because the estimator returns one pose and no uncertainty (Fig.~\ref{fig:teaser}a,b).

We measured what this costs. On LM-Grasp, which pairs 15{,}800 real LINEMOD RGB-D images with 16 candidate grasps each and with outcomes simulated on the scanned meshes at the annotated poses (Fig.~\ref{fig:method}), grasps planned on the FoundationPose estimate succeed in 0.791 of test scenes, against 0.856 when the same planner receives the annotated pose. The loss grows with the error inside the accepted range: on scenes whose estimate is off by 5 to 10\% of the diameter, all of which pass ADD(-S), planned grasps succeed 0.796 of the time against 0.984 with the annotated pose (Fig.~\ref{fig:teaser}c). The planner's confidence is equally misleading: the quarter of scenes on which it is most confident succeed only 0.902 of the time.

Three lines of prior work bear on this problem, and none closes it. First, model-free grasp detectors predict grasp quality directly from images or point clouds~\cite{mahler2017dexnet,morrison2018ggcnn,tenpas2017gpd,sundermeyer2021contact,fang2023anygrasp,newbury2023review}; they avoid pose error but cannot use the object model, and on our data they remain well below model-based planning (Table~\ref{tab:main}). Second, grasp planning under pose uncertainty is well studied when the uncertainty is supplied as a distribution over pose errors~\cite{weisz2012poseerror,hsiao2011robust,mahler2016dexnet1}, but it assumes that such a distribution is available. Third, several estimators output pose distributions directly~\cite{deng2021poserbpf,murphy2021ipdf,haugaard2023spyropose,shi2021fastuq}, but the most accurate render-and-compare estimators, such as FoundationPose, return a single pose, and replacing them would discard their accuracy. The missing piece is therefore a calibrated pose distribution for an accurate estimator that reports none. We build this distribution around the estimator instead of inside it. Its errors are measured on annotated training frames, and for a new frame they are resampled from the training frames whose observable fit (how well the estimated model explains the depth and the mask, and from which side the object is seen) is most alike, a nearest-neighbour residual bootstrap~\cite{efron1979bootstrap,lall1996nnboot}. Each candidate grasp is scored by its mean success across the resulting pose samples and four physics scenarios. We also evaluate two-view fusion using annotation-derived camera motion as an oracle replay; this is not a robot-executable active re-observation policy.

We ask four research questions:
\begin{itemize}
\item[\textbf{RQ1}] How much grasp success does a planner lose by trusting a pose that passes ADD(-S)?
\item[\textbf{RQ2}] Can a posterior built from the estimator's own residuals recover that loss, and which property of the posterior matters?
\item[\textbf{RQ3}] Does planning on the posterior tell the robot when to act?
\item[\textbf{RQ4}] How much does oracle-motion two-view fusion help, and does the posterior transfer to unseen, heavily occluded scenes?
\end{itemize}
Our contributions answer them in order.
\begin{enumerate}
\item \textbf{Measurement of the pose-to-grasp gap on real images with simulated outcomes (RQ1).} Errors inside the ADD(-S) tolerance remove a share of successful grasps that grows with the error; FoundationPose-planned grasps lose 0.065 in top-1 success (95\% scene-wise paired-bootstrap interval [+0.055, +0.076]) relative to the annotated pose.
\item \textbf{A pose posterior for estimators that report none (RQ2).} A feature-conditioned residual bootstrap that requires no retraining, is close to calibrated (coverage 0.87 in rotation and 0.84 in translation at the 0.9 level) and transfers across objects.
\item \textbf{Pose-posterior grasp planning (RQ2, RQ3).} The observed top-1 success is close to that of the annotated-pose planner; the interval is compatible with a small shortfall, and no equivalence test was performed. Matched controls separate the posterior's effect from Monte Carlo averaging and generic Gaussian smoothing.
\item \textbf{Oracle-motion two-view replay and transfer (RQ4).} A confidence-triggered replay using annotation-derived camera motion, plus a transfer check on 1{,}474 object instances in cluttered LM-O scenes~\cite{brachmann2014lmo}. This replay is not a deployable active-view evaluation. Code that regenerates every number in this article will be released with it.
\end{enumerate}
The Methods section describes the pose posterior, the planner and the LM-Grasp benchmark; the Experimental results section answers RQ1 to RQ4 in order; and the Discussion summarises the implications and the scope of the study.

\begin{figure*}[t]
\centering
\includegraphics[width=\textwidth]{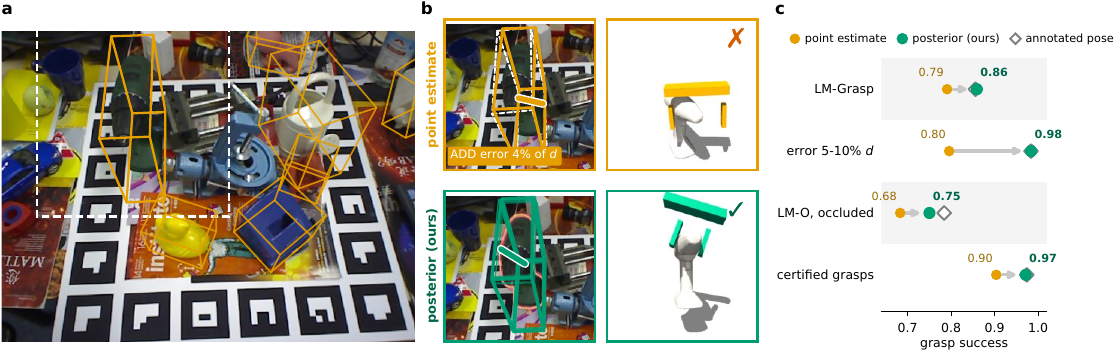}
\caption{\textbf{An accurate pose estimate is not enough to grasp.} \textbf{a},~A real LM-O test image with the FoundationPose estimate of every annotated object (orange). \textbf{b},~The boxed object. Its estimate passes the ADD(-S) test (error below 10\% of the diameter $d$; dashed: annotation), and the planner using it scores its grasp (orange) at 1, yet the grasp drops the object when executed in MuJoCo at the annotated pose with the scene's physics. Planning on 16 samples of our pose posterior (teal boxes; heat: pixels where the samples disagree about the object's extent) assigns the grasp no singleton-success conformal set and selects another that succeeds in simulation. \textbf{c},~Simulated grasp success when planning on the FoundationPose estimate, on our posterior and on the annotated pose: top-1 over the 2{,}579 LM-Grasp test scenes, over the 186 of them whose estimate is off by 5 to 10\% of $d$, over 1{,}474 occluded LM-O instances, and among test scenes assigned singleton-success conformal sets at $\alpha=0.1$ on LM-Grasp. Example chosen by a fixed rule (Methods).}\label{fig:teaser}
\end{figure*}

\section{Methods}\label{sec:methods}

\subsection{Overview}
Fig.~\ref{fig:method} summarises the method on one test scene. FoundationPose returns a pose $(\hat R,\hat t)$. Observable statistics of how well this pose explains the depth image and the mask, together with the viewing direction, locate the frame among the annotated training frames of the same object. The FoundationPose residuals of its $k=64$ nearest neighbours then turn the estimate into $M=16$ pose samples. Every candidate grasp is simulated at each sample under four equally weighted physics scenarios and scored by its mean success over these 1,024 rollouts per scene (16 candidates $\times$ 16 samples $\times$ 4 draws). The planner selects the best candidate and abstains if it is not assigned a singleton-success conformal prediction set. The second-view result is an offline oracle-motion replay triggered by a low score, not a robot-executable request to acquire a view. The following subsections describe each component, the LM-Grasp benchmark, the baselines and the evaluation protocol.

\begin{figure*}[t]
\centering
\includegraphics[width=\textwidth]{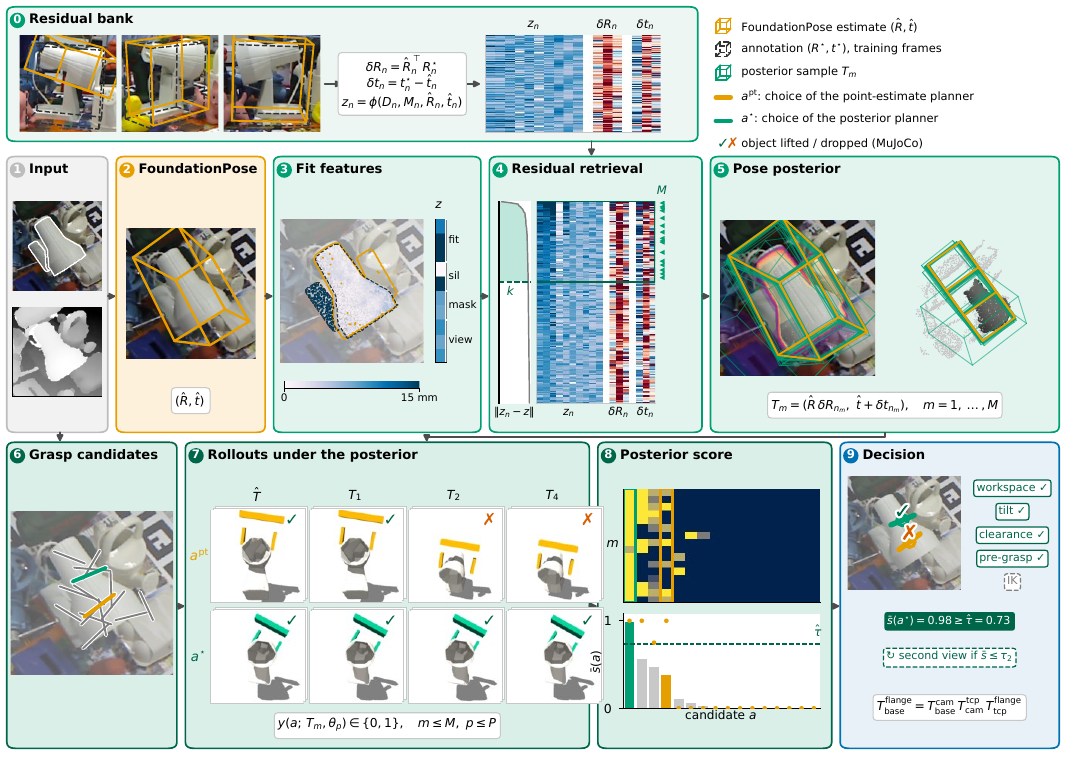}
\caption{\textbf{Pose-posterior grasp planning, step by step on one LM-Grasp test scene.} (0)~Offline, each annotated training frame $n$ stores the FoundationPose residual $(\delta R_n,\delta t_n)$ and ten observable features $z_n$ (thumbnails: three bank frames this scene draws from, estimate orange, annotation dashed; image: one row per training frame of the object, features standardised, residual components in red and blue by sign). (1)~RGB-D input with the visible mask. (2)~FoundationPose, official weights, frozen. (3)~Fit features: mask depth points coloured by their distance to the posed model surface, the projected model silhouette (orange) against the mask (dashed), and the feature vector $z$. (4)~Bank frames of the object sorted by feature distance to $z$ (frames within one 50-frame block of the query excluded); the $k=64$ nearest (box) form the pool from which $M=16$ residuals are drawn (arrows). (5)~The drawn residuals applied to the estimate: pose samples $T_m$ on the image, with the heat of their disagreement about the object's extent, and on the depth point cloud. (6)~Sixteen candidates proposed from the depth points and the fitted table plane alone. (7)~Every candidate is executed in MuJoCo at every sample under $P=4$ physics draws; renders show the nominal-physics rollout after the lift for the point-estimate planner's choice $a^{\mathrm{pt}}$ and the posterior planner's choice $a^{\star}$, at the estimate $\hat T$ and at three samples. (8)~Success under every sample (columns sorted by $\bar s$) and the posterior mean $\bar s(a)$; dots: the point-estimate score, which is 1 for $a^{\mathrm{pt}}$. (9)~The selected candidate receives a singleton-success conformal prediction set according to the calibration procedure in Methods ($\alpha=0.1$); this is not a guarantee on conditional precision. A low $\bar s(a^{\star})$ triggers the offline oracle-motion replay. Marks in (9): outcome at the annotated pose.}\label{fig:method}
\end{figure*}

\subsection{Pose estimation}
FoundationPose~\cite{wen2024foundationpose} runs with its official weights on the full RGB-D image, the visible mask, the intrinsics and the mesh, with five refinement iterations (its LINEMOD setting) and the BOP symmetries of eggbox and glue (median 0.57\,s per image on one NVIDIA RTX PRO 6000). We use the object-to-camera convention $T_{co}=\begin{bmatrix}R&t\\0&1\end{bmatrix}$, so $R$ maps object-frame coordinates into the camera frame and $t$ is expressed in metres in the camera frame.

\subsection{Residual pose posterior}
For a training frame with estimate $(\hat R,\hat t)$ and annotation $(R,t)$ the residual is
\begin{equation}
\delta R=\hat R^{\top}R^{\star},\qquad \delta t=t^{\star}-\hat t,
\end{equation}
where $(R^{\star},t^{\star})$ is the symmetry-equivalent annotation minimizing the geodesic rotation angle $\arccos\!\left(\operatorname{clip}\left((\operatorname{tr}(\hat R^{\top}R^{\star})-1)/2,-1,1\right)\right)$ over the BOP symmetry set; the principal angle lies in $[0,\pi]$. Rotation residuals live in the object frame, where shape ambiguities stay attached to the object; translation residuals live in the camera frame, where depth bias stays along the viewing ray, and $\delta t$ is in metres. Ten features are computed from the observation and the estimate only: the median and 90th-percentile distance of up to 3,000 back-projected mask pixels to the posed mesh surface, the fraction farther than 10\,mm, the fraction of the projected model silhouette covered by the mask, the fraction of the mask outside it, the valid-depth fraction, the log mask area, and the three components of the camera direction in the estimated object frame. For a frame of object $o$, features are standardised componentwise by the per-object bank median and interquartile range (all 8{,}942 frames of the LM-Grasp training split); zero-IQR components are omitted from the distance calculation, and neighbours are ranked by Euclidean distance in the remaining standardised feature space. Then $M=16$ residuals are drawn without replacement from the $k=64$ nearest bank frames of $o$, a nearest-neighbour residual bootstrap~\cite{efron1979bootstrap,lall1996nnboot}; pose samples are $T_m=(\hat R\,\delta R_m,\ \hat t+\delta t_m)$. Bank frames of the same object within one 50-frame block of the query are excluded in every split, so no frame draws its own residual or a near-duplicate's. Controls with the same budget: a Gaussian with each object's root-mean-square residual angle about a random axis and per-axis translation s.d.; unconditional resampling of the object's residuals; and the $k$ nearest frames of the other twelve objects on the seven object-agnostic features.

\subsection{Planning under the posterior}
A candidate $a$ is simulated with the object at every sample $T_m$ under $P=4$ physics scenarios $\theta_p$ (one nominal setting and three prior samples, common to all candidates and samples) and scored by
\begin{equation}
\bar s(a)=\frac{1}{MP}\sum_{m=1}^{M}\sum_{p=1}^{P} y\big(a;\,T_m,\theta_p\big),
\end{equation}
where $y\in\{0,1\}$ is simulated success and a colliding rollout counts as a failure. The four scenarios receive equal weight; thus $\bar s$ is the mean over this finite scenario set, not an unbiased Monte Carlo estimate of success under the full stated physics prior. The evaluated planner selects the highest $\bar s$; the point-estimate planner is the case $M=1$ at $(\hat R,\hat t)$, the annotated-pose planner uses $(R,t)$.

\subsection{Second-view replay (oracle motion)}
For each test frame, the paired view is selected offline from its 50-frame block as the frame whose annotation-derived viewing direction differs most from the current view, up to $30^\circ$ (median 29$^\circ$). We use $T_{ab}$ to denote a rigid transform mapping coordinates from frame $b$ into frame $a$. The second-view estimate and posterior samples are mapped into the current camera with $T_{c_i c_j}=T_{c_i o}T_{c_j o}^{-1}$, computed from the annotated object poses. This is an oracle-motion, retrospective replay, not a deployable active-view policy: it assumes annotation-derived camera motion and does not model robot motion, sensing time or motion risk. The planner uses the 32 samples from both views. A scene is replayed when $\bar s$ for its selected grasp is at most a threshold fixed to select 25\% of validation scenes; applied unchanged to test, this selects 22\%. The reported extra pose-estimation time excludes view acquisition and robot motion.

\subsection{LM-Grasp benchmark}
\paragraph{Images and splits.} We use every real test image of the 13 objects of the standard LINEMOD protocol~\cite{hinterstoisser2012} in the BOP release~\cite{hodan2018bop}, 15{,}800 RGB-D images with one annotated target in clutter. Images are split per object into blocks of 50 consecutive frames assigned to train, validation, calibration or test (probabilities 0.55, 0.10, 0.15, 0.20; 8{,}942, 1{,}679, 2{,}600 and 2{,}579 scenes), so near-duplicate frames never straddle two splits. The reported bootstrap resamples individual scenes, not these temporal blocks. Every method receives the annotated visible mask, the known-detection assumption of LINEMOD pose studies; learned scorers see a $96\times96$ RGB-D crop around it.

\paragraph{Candidates and outcomes.} Each scene receives 16 parallel-jaw candidates built from the observation alone: a visible surface point (with probability 0.7 from the top 5\,cm), an approach tilted by up to $45^\circ$ from the normal of a table plane fitted to the depth image, 0 to 3\,cm of penetration, at least 4\,cm of table clearance and a random closing angle. Each candidate is mapped into the object frame with the annotated pose and executed in MuJoCo~\cite{todorov2012mujoco} against a CoACD~\cite{wei2022coacd} decomposition of the scanned mesh standing on a table (the model $+z$ axis is within a per-object median of 0.6 to 1.4$^\circ$ of the fitted table normal). A floating gripper (138\,mm opening, 70\,mm fingers, 10\,N per finger) closes and lifts by 12\,cm; success means the object ends at least 5\,cm above its rest height, and a candidate whose open jaws already intersect the object or the table fails. Friction (log-normal, median 0.5), mass (log-normal around the mesh volume times 500\,kg\,m$^{-3}$) and centre of mass (Gaussian, s.d. 0.15 of the half-extents) are drawn per scene and are invisible to every method; LINEMOD publishes no physical properties, so they are a prior. The corpus holds 252{,}800 rollouts; on the test split 0.266 of candidates succeed and 0.877 of scenes contain a success.

\subsection{Learned baselines and fusion}
The learned scorers share an ImageNet-initialised ResNet-18~\cite{he2016resnet,deng2009imagenet} with an RGB-D stem and an outcome head in which a 32-dimensional image code modulates each candidate's action embedding; the deep ensemble~\cite{lakshminarayanan2017} uses five members per run, with three independently trained ensemble runs. The Dex-Net 2.0 grasp-quality architecture~\cite{mahler2017dexnet} scores a depth patch aligned with each candidate. A pose regressor (six-dimensional rotation~\cite{zhou2019rotation} and translation) feeds the planner. POISE and POISE-U append the point-estimate or posterior planner score to each candidate's features. All are trained on the same data for 20 epochs with AdamW~\cite{loshchilov2019adamw} (Adam for the Dex-Net architecture), identical augmentation and checkpoint selection on validation loss; single-model methods use five independent training seeds.

\subsection{LM-O transfer check}
We use the 200 BOP test frames of LM-O~\cite{brachmann2014lmo} and every annotated instance with visible fraction at least 0.1. Observations, candidates, physics prior, simulator and FoundationPose settings follow LM-Grasp, with outcomes simulated at the LM-O annotation, and the residual bank is the LM-Grasp training split. Because the simulator stands objects upright, we also report the instances whose model $+z$ axis is within $15^\circ$ of the fitted table normal (0.558 of instances).

\subsection{Selection protocol}
Every design choice (posterior features, $k$, residual frames, controls, second-view rule, fusion variants) was fixed in a written plan before it was run and selected on the validation split; the posterior planner, its controls and the fusion variants were then evaluated on the test split once, in a single pass (baseline test results had been computed earlier and were not revisited). The LM-O protocol, including the visibility cut and the upright subset, was fixed before any LM-O result was computed. Fig.~\ref{fig:qualitative} shows the two LM-O frames in which the posterior's choice succeeds and the point estimate's fails for the most objects, plus the frame with the most objects where the reverse holds. The example in Fig.~\ref{fig:teaser} is, among LM-O instances in other frames, the one with the largest visible area whose estimate passes ADD(-S), whose point-estimate planner rates its choice 1 but fails, and whose posterior score assigns that choice no singleton-success conformal set and selects another choice that succeeds in simulation.

\subsection{Metrics and statistics}
\emph{Top-1 success} is the outcome of the highest-scoring candidate averaged over scenes (ties broken by a fixed jitter shared by all methods). \emph{Within-scene AUC} is the ROC area among a scene's candidates, averaged over the 2{,}261 test scenes with both outcomes. \emph{Selective success} at coverage $c$ is the top-1 success on the fraction $c$ of scenes with the highest score of the selected candidate. For each method, let $\hat p(a)$ be its predicted success probability for candidate $a$ (for pose planners, $\hat p(a)=\bar s(a)$). In each calibration scene the method selects a candidate without using its outcome; let $y_i\in\{0,1\}$ be its observed outcome and define $r(a_i,1)=1-\hat p(a_i)$ and $r(a_i,0)=\hat p(a_i)$, with observed calibration score $r_i=r(a_i,y_i)$. For nominal miscoverage level $\alpha$ and $n_{\rm cal}$ calibration scenes, set $k=\lceil(n_{\rm cal}+1)(1-\alpha)\rceil$ and let $q_\alpha$ be the $k$th ordered calibration score, taking $q_\alpha=+\infty$ if $k>n_{\rm cal}$. The prediction set for a selected test candidate is $\Gamma_\alpha(a)=\{y\in\{0,1\}:r(a,y)\le q_\alpha\}$; ties are included. We report the held-out fraction assigned the singleton set $\{1\}$ and its empirical success precision. Standard split-conformal marginal label-set coverage requires exchangeability of calibration and test units and does not imply precision conditional on issuing a singleton-success set. Because frames within 50-frame blocks are temporally related, that exchangeability assumption is not established here; these are descriptive test-set estimates, not formal selective guarantees. No block-aware intervals for singleton fractions or precision are available. Pose accuracy is the ADD (ADD-S for eggbox and glue~\cite{hodan2016eval}) pass rate at 10\% of the diameter. Posterior calibration is the fraction of scenes whose realised error lies below a quantile of the posterior samples' errors. Differences in top-1 success are paired bootstraps over individual scenes (10,000 resamples), seed-averaged for learned methods, reported with 95\% percentile intervals; these intervals do not account for temporal dependence within blocks and may understate uncertainty. Learned methods report mean $\pm$ s.d. over independent runs. The primary comparisons (posterior versus point estimate, and posterior versus annotated pose) were specified before the test split was evaluated; the remaining comparisons are supporting analyses, reported without multiplicity adjustment. Software: Python 3.12, PyTorch 2.14, MuJoCo 3.10; FoundationPose in its own environment (PyTorch 2.7.1, CUDA 12.8); two-view planning costs 20.5 CPU-seconds per scene.

\section{Experimental results}\label{sec:results}
We answer the four research questions on the LM-Grasp test split (2{,}579 scenes) and, for transfer, on LM-O. All grasp outcomes are simulated. Differences are paired across scenes and reported with 95\% scene-wise bootstrap intervals (Methods); they do not account for temporal dependence within blocks.

\subsection{Pose accuracy is a poor proxy for grasp success (RQ1)}
\textbf{Answer to RQ1: trusting a FoundationPose estimate that passes ADD(-S) costs 0.065 in simulated top-1 grasp success, and the cost grows with the pose error inside the accepted range.}
The planner given the annotated pose chooses a successful grasp in 0.856 of test scenes, close to the 0.877 of scenes that contain any successful candidate. Perturbing that pose shows how quickly grasps fail while ADD(-S) is still satisfied (Fig.~\ref{fig:posegap}a,b): a rotation error with s.d. $5^\circ$ leaves the ADD(-S) pass rate at 1.000 and lowers top-1 success to 0.809; a translation error of 5\,mm per axis passes on 0.951 of scenes and lowers success to 0.759. FoundationPose operates in this regime (median errors 1.9$^\circ$ and 4.4\,mm). Its planned grasps succeed in 0.791 of scenes, 0.065 below the annotated pose (95\% CI [+0.055, +0.076]), and the loss grows with its error inside the accepted range (Fig.~\ref{fig:posegap}c).

\begin{figure*}[t]
\centering
\includegraphics[width=\textwidth]{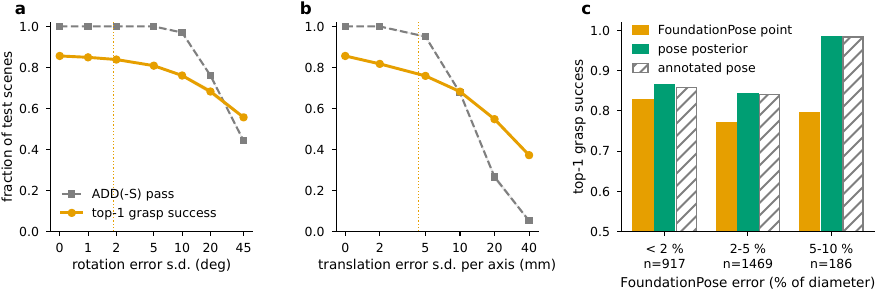}
\caption{\textbf{Pose error versus simulated grasp success on LM-Grasp test scenes.} \textbf{a}, \textbf{b}, The annotated pose perturbed by a random rotation (\textbf{a}) or translation (\textbf{b}) before planning: top-1 grasp success (orange) and the ADD(-S)$<0.1d$ pass rate (dashed); bands: $\pm$ s.d. over three noise seeds; dotted line: median FoundationPose error. \textbf{c}, Top-1 success binned by the FoundationPose ADD(-S) error relative to the object diameter, for planning on the point estimate, on our posterior and on the annotated pose; n: scenes per bin (7 scenes above 10\% not shown).}\label{fig:posegap}
\end{figure*}

\subsection{The posterior recovers annotated-pose grasp success (RQ2)}
\textbf{Answer to RQ2: posterior planning substantially improves simulated top-1 success, with an observed result close to annotated-pose planning. The empirical residual shape, rather than extra rollouts or matched variance alone, best explains the gain.}
Planning on the posterior raises simulated top-1 success from 0.791 to 0.859 (Table~\ref{tab:main}, Fig.~\ref{fig:results}a), compared with 0.856 for the annotated pose (difference +0.003; 95\% scene-wise paired-bootstrap interval [$-$0.003, +0.009]). This interval is compatible with no difference and a small shortfall; it is not an equivalence test. The gain is largest exactly where the point estimate fails: on the 186 scenes with errors between 5\% and 10\% of the diameter, success rises from 0.796 to 0.984, the annotated-pose value, and on the 1{,}469 scenes between 2\% and 5\% from 0.770 to 0.841 (Fig.~\ref{fig:posegap}c). Of the 730 scenes in which the two planners choose different grasps, the posterior's choice succeeds while the point estimate's fails in 185, and the reverse happens in 9. A second sampling seed of the posterior gives 0.861.

Table~\ref{tab:ablation} rules out two simpler explanations. Monte Carlo averaging alone (the point estimate with the same 1,024 rollouts per scene) reaches 0.814, 0.045 below the posterior (CI [+0.036, +0.054]). Generic smoothing (a Gaussian with the residuals' second moments) reaches 0.849 (difference +0.010, CI [+0.005, +0.017]) and is poorly calibrated in translation (coverage 0.687 at 0.9). What matters is the empirical shape of the errors: unconditional resampling of the object's residuals reaches 0.857 (difference +0.002, CI [$-$0.003, +0.007]), and residuals of the twelve other objects still reach 0.854, 0.040 above the matched point estimate. The feature conditioning tells the robot which frames are uncertain (rank correlation of predicted spread and realised error 0.40 in rotation and 0.65 in translation, against 0.32 and 0.38 without it; Fig.~\ref{fig:analysis}a,b) while matching unconditional resampling in top-1 success.

A learned scorer that fuses the image with the planner's score does not add to the posterior planner: the scorer with the posterior score as input (POISE-U) reaches 0.859, a difference of 0.000 to the planner alone (CI [$-$0.003, +0.003]), although the same fusion helps the point-estimate planner (0.821). Learned outcome scorers without an object model stay well below the FoundationPose-based planners on top-1 success (0.705 for a deep ensemble).

\begin{table*}[t]
\caption{\textbf{Matched comparison on the LM-Grasp test split} (2{,}579 scenes). $n$: independent training seeds for single-model methods, independent ensemble runs for the deep ensemble, or repeated runs; learned entries are mean $\pm$ s.d. across runs. Top-1: simulated success of the highest-scoring candidate. AUC: within-scene ranking AUC. Sel@25\%: success on the 25\% of scenes with the most confident choice. Singleton success: held-out fraction assigned the singleton set $\{1\}$ / observed precision at $\alpha=0.1$; these are empirical estimates, not selective guarantees, and no block-aware intervals were computed. The triggered-replay row reports its top-1 gain over one-view posterior planning (+0.004; 95\% scene-wise interval [+0.001, +0.007]), not its absolute success; the all-scenes replay has gain +0.006, with no interval available in the reported results. Shading marks the best and second-best value among non-privileged methods (ties at three decimals shaded alike); our rows are tinted. Higher is better for performance columns; singleton fraction is not shaded. Pose-then-plan methods use the object's mesh, the simulator and the four-scenario physics design at decision time.}\label{tab:main}
\centering\footnotesize
\setlength{\tabcolsep}{3.2pt}
\begin{tabular}{lccccc}
\toprule
Method & n & Top-1 & AUC & Sel@25\% & Success set (frac./prec.)\\
\midrule
\multicolumn{6}{l}{\textcolor{famlearn}{\textbf{Learned outcome scorers (no object model)}}}\\
Dex-Net 2.0 architecture\cite{mahler2017dexnet} & 5 & 0.637 $\pm$ 0.013 & 0.829 & 0.902 & 0.611 / 0.831\\
RGB-D outcome scorer & 5 & 0.683 $\pm$ 0.006 & 0.876 & 0.929 & 0.647 / 0.857\\
Deep ensemble (5 members)\cite{lakshminarayanan2017} & 3 & 0.705 $\pm$ 0.006 & 0.891 & 0.938 & 0.664 / 0.884\\
\midrule
\multicolumn{6}{l}{\textcolor{famfuse}{\textbf{Learned fusion with the planner score}}}\\
POISE (scorer + point score) & 5 & 0.821 $\pm$ 0.002 & 0.958 & 0.963 & 0.874 / 0.927\\
POISE-U (scorer + posterior score) & 5 & \second 0.859 $\pm$ 0.002 & \second 0.988 & 0.985 & 0.851 / \best 0.972\\
\midrule
\multicolumn{6}{l}{\textcolor{fampose}{\textbf{Pose-then-plan (mesh, simulator and physics prior at decision time)}}}\\
Pose regressor + planner & 5 & 0.707 $\pm$ 0.007 & 0.859 & 0.817 & n/a\\
FoundationPose\cite{wen2024foundationpose} + planner & 1 & 0.791 & 0.935 & 0.902 & 0.844 / 0.904\\
FoundationPose + planner, 64 draws & 1 & 0.814 & 0.951 & 0.919 & n/a\\
\rowcolor{oursrow}\textbf{Pose posterior (ours)} & 1 & \second 0.859 & 0.987 & \second 0.988 & 0.849 / \best 0.972\\
\rowcolor{oursrow}\textbf{Triggered (22\%, oracle; gain)} & 1 & +0.004 & n/a & n/a & n/a\\
\rowcolor{oursrow}\textbf{All scenes (oracle)} & 1 & \best 0.865 & \best 0.993 & \best 0.997 & n/a\\
\midrule
\multicolumn{6}{l}{\textcolor{famref}{\textbf{Privileged references}}}\\
Annotated pose + planner & 1 & 0.856 & 0.990 & 0.984 & 0.854 / 0.974\\
Oracle (best candidate) & 1 & 0.877 & n/a & n/a & n/a\\
\bottomrule
\end{tabular}
\end{table*}

\begin{table*}[t]
\caption{\textbf{What the posterior needs} (LM-Grasp test split). Every row plans the same 16 candidates with common physics draws; rollouts are per scene. Cov@0.9: fraction of scenes whose FoundationPose rotation / translation error lies below the distribution's 90\% quantile (calibrated: 0.90). $\rho$: Spearman correlation between the predicted spread of a frame and its realised error. $\Delta$: paired difference in top-1 success to the point estimate with the same 1,024 rollouts.}\label{tab:ablation}
\centering\footnotesize
\setlength{\tabcolsep}{2.6pt}
\begin{tabular}{lccccccc}
\toprule
 & \multicolumn{2}{c}{Coverage@0.9} & \multicolumn{2}{c}{$\rho$(spread, error)} & & & \\
\cmidrule(lr){2-3}\cmidrule(lr){4-5}
Pose distribution & rot. & trans. & rot. & trans. & Top-1 & $\Delta$ vs 64 draws [95\% CI] & Sel@25\%\\
\midrule
Point estimate, 4 draws & n/a & n/a & n/a & n/a & 0.791 & $-$0.023 [$-$0.031, $-$0.016] & 0.902\\
Point estimate, 64 draws & n/a & n/a & n/a & n/a & 0.814 & ref. & 0.919\\
Gaussian, matched moments & 0.92 & 0.69 & 0.24 & 0.25 & 0.849 & +0.035 [+0.026, +0.043] & 0.981\\
Bootstrap, unconditional & 0.86 & 0.85 & 0.32 & 0.38 & 0.857 & +0.043 [+0.034, +0.052] & 0.988\\
kNN, other 12 objects & 0.85 & 0.81 & 0.10 & 0.44 & 0.854 & +0.040 [+0.031, +0.049] & 0.981\\
\rowcolor{oursrow}\textbf{kNN, same object (ours)} & 0.87 & 0.84 & 0.40 & 0.65 & 0.859 & +0.045 [+0.036, +0.054] & 0.988\\
\bottomrule
\end{tabular}
\end{table*}

\begin{figure*}[t]
\centering
\includegraphics[width=\textwidth]{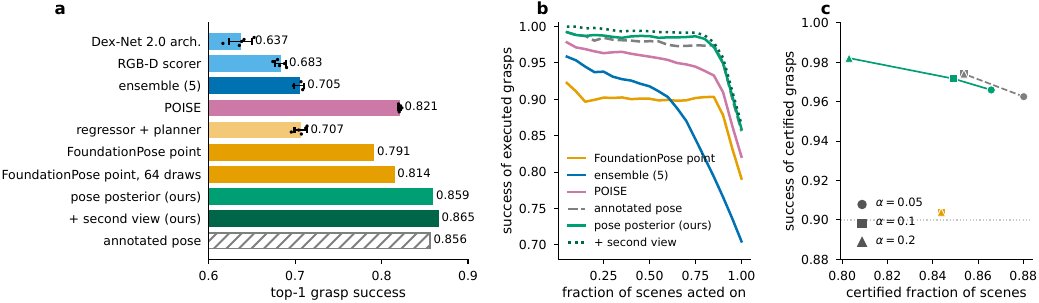}
\caption{\textbf{Simulated grasp success and when to act.} \textbf{a}, Top-1 simulated grasp success on the test split; dots: independent training runs; error bars: s.d.; light orange: learned pose regressor; hatched: point estimate with the matched rollout budget, and the annotated pose; the all-scenes second-view result uses oracle motion. \textbf{b}, Success of the executed grasps when a method acts only on its most confident fraction of scenes (seed-averaged for learned methods). \textbf{c}, Split-conformal prediction sets at $\alpha\in\{0.05,0.1,0.2\}$: fraction assigned singleton-success sets and their held-out precision; the point-estimate planner assigns none at $\alpha=0.05$. The plotted ``certified'' label denotes this set event, not a conditional-precision guarantee; precision is empirical.}\label{fig:results}
\end{figure*}

\subsection{Knowing when to act (RQ3)}
\textbf{Answer to RQ3: the posterior score better ranks scenes by observed risk, and its singleton-success predictions have high empirical held-out precision.}
The point-estimate planner averages physics scenarios at a pose it treats as exact and therefore omits pose uncertainty. On the quarter of scenes where it is most confident, posterior planning succeeds 0.988 of the time, against 0.902 for the point estimate and 0.938 for the deep ensemble of learned scorers (Fig.~\ref{fig:results}b). Split-conformal prediction sets~\cite{vovk2005,angelopoulos2023} are formed for each planner-selected candidate as described in Methods. At $\alpha=0.1$, the posterior assigns singleton-success sets to 0.849 of test scenes, with empirical held-out precision 0.972; the corresponding annotated-pose values are 0.854 and 0.974, and the point-estimate values are 0.844 and 0.904 (Fig.~\ref{fig:results}c). These precision figures do not have a formal guarantee conditional on issuing a singleton-success set. Because a colliding rollout counts as a failure, candidates whose open jaws would hit the object under some of the plausible poses are also ranked down.

\begin{figure*}[t]
\centering
\includegraphics[width=\textwidth]{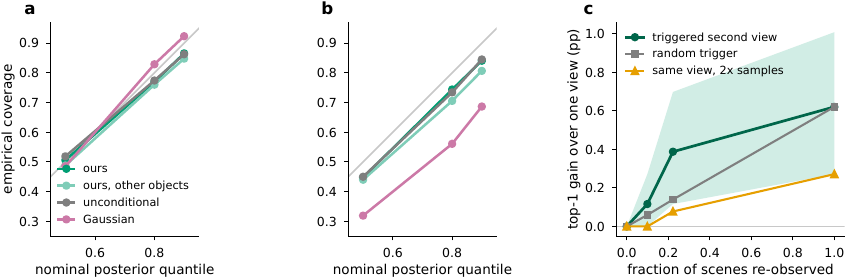}
\caption{\textbf{Posterior calibration and oracle-motion two-view replay.} \textbf{a}, \textbf{b}, Fraction of test scenes whose realised FoundationPose rotation (\textbf{a}) and translation (\textbf{b}) error lies below the posterior's 50, 80 and 90\% quantile, for our posterior, the same posterior built from the other twelve objects, unconditional resampling and a moment-matched Gaussian; diagonal: perfect calibration. \textbf{c}, Simulated top-1 gain in the offline two-view replay when the least-confident scenes are selected by a validation-fixed threshold; comparison conditions are random-scene replay and doubling first-view samples. Camera motion is annotation-derived (oracle), and the interval is a scene-wise bootstrap interval that does not account for block dependence.}\label{fig:analysis}
\end{figure*}

\subsection{Oracle-motion replay and unseen clutter (RQ4)}
\textbf{Answer to RQ4: an offline oracle-motion replay suggests that confidence-triggered two-view fusion can improve simulated success; it does not establish a deployable policy or account for view-acquisition costs. The posterior also transfers to unseen, heavily occluded scenes and recovers about two thirds of the gap there.}
LINEMOD sequences show a static scene from a moving camera. Here the paired frame and camera transform are obtained retrospectively using annotations, so the experiment is an oracle-motion replay rather than a robot acquiring a view. With the validation-fixed confidence threshold, the replay selects 22\% of test scenes and simulated top-1 success rises by +0.004 (95\% scene-wise bootstrap interval [+0.001, +0.007]); this triggered gain exceeds the doubled-first-view-samples control by +0.003 (interval [0.000, +0.006]). Fusing a second view for every scene gives a +0.006 gain, corresponding to the all-scenes table result of 0.865 versus 0.859 for one-view posterior planning. The stated 0.57\,s covers an extra pose estimate only; robot motion, view acquisition time and risk are not included.

LM-O~\cite{brachmann2014lmo} annotates eight LINEMOD objects in 200 frames of a cluttered sequence with heavy mutual occlusion; no LM-O annotation enters any of our components. With the residual bank, the planner settings and the physics prior unchanged, posterior planning reaches 0.750 top-1 success on the 1{,}474 annotated instances, against 0.683 for the point estimate (difference +0.067, 95\% CI [+0.053, +0.081]) and 0.784 with the annotated pose (Fig.~\ref{fig:qualitative}), recovering 66\% of the gap; on its most confident quarter it succeeds 0.959 of the time, against 0.861 for the point estimate and 0.970 with the annotated pose. The remaining difference of 0.034 (CI [+0.021, +0.047]) is consistent with occlusion producing larger and differently shaped errors than the LINEMOD frames from which the residual bank is built, so adding occluded frames to the bank is a direct route to a larger gain. The result holds on the instances standing upright on the table (0.825 against 0.765; difference +0.060, CI [+0.041, +0.079]), where the simulator's upright placement matches the scene exactly.

\begin{figure*}[t]
\centering
\includegraphics[width=\textwidth]{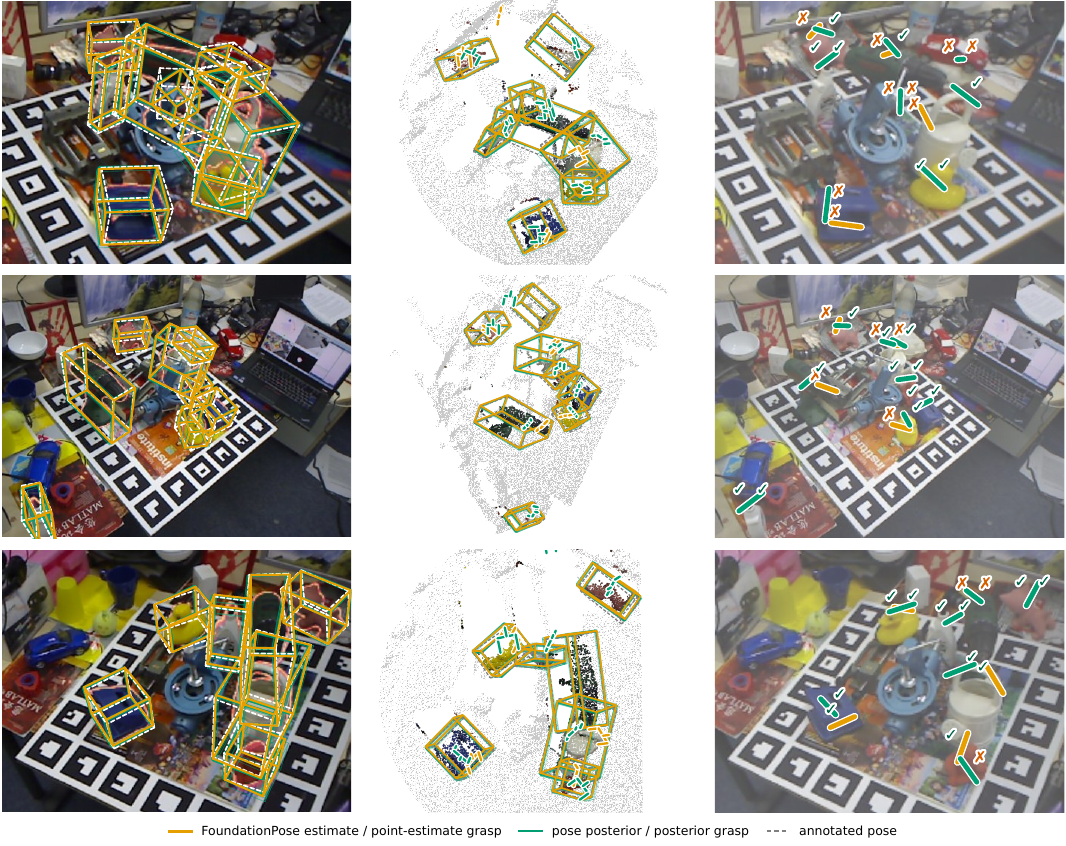}
\caption{\textbf{Qualitative results on unseen cluttered LM-O frames.} Each row is one real test image with all annotated objects. Left: FoundationPose estimates (orange boxes), posterior samples (teal boxes) and annotated poses (dashed), with the heat of posterior disagreement. Middle: the depth point cloud with the same boxes. Right: grasps chosen by point-estimate planning (orange) and posterior planning (teal) for every object, with simulated outcomes (\checkmark, $\times$). Frames chosen by a fixed rule (Methods); the last row includes objects on which the posterior's choice fails.}\label{fig:qualitative}
\end{figure*}

\section{Discussion}\label{sec:discussion}
Pose estimators are judged by how close their pose is to the annotation; a grasp planner also needs to account for pose error. In this simulator-based study, treating the FoundationPose estimate as exact is associated with a 0.065 reduction in top-1 success on real LINEMOD images. Posterior planning raises simulated success to 0.859, compared with 0.856 for the annotated pose (difference +0.003; 95\% scene-wise paired-bootstrap interval [$-$0.003, +0.009]). This interval is compatible with no difference and a small shortfall; because no equivalence margin or test was specified, it does not establish equivalence or complete recovery. The mechanism is the classical one of planning under pose uncertainty~\cite{weisz2012poseerror,hsiao2011robust,mahler2016dexnet1}: grasps that survive plausible pose errors are preferred over grasps that succeed only at the estimate. What has been missing for learned render-and-compare estimators is the distribution itself; its empirical shape matters more than a parametric form with the right variance. Because the posterior is built around the estimator, the approach could be applied to other estimators if their residuals can be measured on annotated frames; the results here do not establish performance for arbitrary estimators or unannotated objects.

A second potential benefit is a more informative confidence score. The point-estimate planner averages physics scenarios at a fixed pose and thus omits pose uncertainty. In the held-out test set, singleton-success conformal predictions have empirical precision 0.904 for the point-estimate planner and 0.972 for posterior planning; these are descriptive estimates, not guarantees conditional on selection. The two-view analysis is an oracle-motion replay and does not yet show that a robot can acquire the view at comparable cost.

The study isolates pose information in a simulator-based setting: images, masks and pose estimates are real, but grasp outcomes are simulated in MuJoCo on scanned meshes. Friction, mass and centre of mass are uncertain because LINEMOD publishes no physical properties. All pose-based planners use the same simulator and physics scenarios, so comparisons are matched; the absolute success rates characterize this simulated setup, not physical-robot performance. The planner gives equal weight to one nominal physics setting and three prior samples, so its score is a finite four-scenario average rather than an unbiased estimate of the full prior expectation. The reported scene-wise bootstrap intervals do not account for dependence among frames in the 50-frame blocks, and no block-aware reanalysis was possible from the available results; interval coverage may therefore be optimistic. Likewise, singleton-success precision is reported empirically, without a selective-precision guarantee or block-aware interval. Every method receives the visible object mask. Posterior planning adds computation (10.3 against 0.7 CPU-seconds per scene) that parallelises across candidates and pose samples. The approach is complementary to estimators that output pose distributions natively~\cite{deng2021poserbpf,haugaard2023spyropose}, but the results here support only the tested FoundationPose/LINEMOD and LM-O setting. Building the residual bank from occluded frames is a hypothesis for future work, not a result established by this transfer test.

\backmatter

\bmhead{Data availability}
LINEMOD and LM-O are available from the BOP benchmark (\url{https://bop.felk.cvut.cz}). The LM-Grasp corpus is generated deterministically from LINEMOD by the released scripts; the generated corpus, pose estimates, posterior samples and all results files are available from the corresponding author on reasonable request.

\bmhead{Code availability}
The code that builds LM-Grasp, computes the pose posteriors, plans and evaluates every method and regenerates every figure, table and number in this article will be released at \url{https://github.com/s-elim/msp-framework} upon publication.

\bmhead{Acknowledgements}
This work was supported by the Institute of Information \& Communications Technology Planning \& Evaluation (IITP) grant funded by the Korea government (MSIT) (IITP-2026-RS-2026-25615195, AI Star Fellowship Support Program). This work was supported by the National Research Foundation of Korea (NRF) grant funded by the Korea government (MSIT) (No.~IRIS RS-2023-00219725). This work was supported by the 2026 Yeungnam University Research Grants.

{\hbadness=2000
\bmhead{Author contributions}\par
M.S.S.: conceptualization; development of the pose-posterior grasp planning algorithm and the LM-Grasp benchmark; software; all experiments and analyses; figures; writing of the original draft. S.K.: supervision; research guidance; proofreading and revision of the manuscript. Both authors reviewed the manuscript.\par}

\bmhead{Competing interests}
The authors declare no competing interests.

\bibliography{refs}

@inproceedings{hinterstoisser2012,
  author    = {Hinterstoisser, Stefan and Lepetit, Vincent and Ilic, Slobodan and Holzer, Stefan and Bradski, Gary and Konolige, Kurt and Navab, Nassir},
  title     = {Model Based Training, Detection and Pose Estimation of Texture-Less {3D} Objects in Heavily Cluttered Scenes},
  booktitle = {Computer Vision -- ACCV 2012},
  pages     = {548--562},
  year      = {2013},
  doi       = {10.1007/978-3-642-37331-2_42}
}

@inproceedings{hodan2018bop,
  author    = {Hodan, Tomas and Michel, Frank and Brachmann, Eric and Kehl, Wadim and Buch, Anders Glent and Kraft, Dirk and Drost, Bertram and Vidal, Joel and Ihrke, Stephan and Zabulis, Xenophon and Sahin, Caner and Manhardt, Fabian and Tombari, Federico and Kim, Tae-Kyun and Matas, Jiri and Rother, Carsten},
  title     = {{BOP}: Benchmark for {6D} Object Pose Estimation},
  booktitle = {Computer Vision -- ECCV 2018},
  pages     = {19--35},
  year      = {2018},
  doi       = {10.1007/978-3-030-01249-6_2}
}

@inproceedings{hodan2016eval,
  author    = {Hodan, Tomas and Matas, Jiri and Obdrzalek, Stepan},
  title     = {On Evaluation of {6D} Object Pose Estimation},
  booktitle = {Computer Vision -- ECCV 2016 Workshops},
  pages     = {606--619},
  year      = {2016},
  doi       = {10.1007/978-3-319-49409-8_52}
}

@inproceedings{wen2024foundationpose,
  author    = {Wen, Bowen and Yang, Wei and Kautz, Jan and Birchfield, Stan},
  title     = {{FoundationPose}: Unified {6D} Pose Estimation and Tracking of Novel Objects},
  booktitle = {Proc. IEEE/CVF Conference on Computer Vision and Pattern Recognition (CVPR)},
  pages     = {17868--17879},
  year      = {2024},
  doi       = {10.1109/CVPR52733.2024.01692}
}

@inproceedings{xiang2018posecnn,
  author    = {Xiang, Yu and Schmidt, Tanner and Narayanan, Venkatraman and Fox, Dieter},
  title     = {{PoseCNN}: A Convolutional Neural Network for {6D} Object Pose Estimation in Cluttered Scenes},
  booktitle = {Robotics: Science and Systems XIV},
  year      = {2018},
  doi       = {10.15607/RSS.2018.XIV.019}
}

@inproceedings{tremblay2018dope,
  author    = {Tremblay, Jonathan and To, Thang and Sundaralingam, Balakumar and Xiang, Yu and Fox, Dieter and Birchfield, Stan},
  title     = {Deep Object Pose Estimation for Semantic Robotic Grasping of Household Objects},
  booktitle = {Proc. Conference on Robot Learning (CoRL)},
  series    = {Proceedings of Machine Learning Research},
  volume    = {87},
  pages     = {306--316},
  year      = {2018},
  eprint    = {1809.10790}
}

@inproceedings{deng2020selfsup,
  author    = {Deng, Xinke and Xiang, Yu and Mousavian, Arsalan and Eppner, Clemens and Bretl, Timothy and Fox, Dieter},
  title     = {Self-Supervised {6D} Object Pose Estimation for Robot Manipulation},
  booktitle = {Proc. IEEE International Conference on Robotics and Automation (ICRA)},
  pages     = {3665--3671},
  year      = {2020},
  doi       = {10.1109/ICRA40945.2020.9196714}
}

@inproceedings{mahler2017dexnet,
  author    = {Mahler, Jeffrey and Liang, Jacky and Niyaz, Sherdil and Laskey, Michael and Doan, Richard and Liu, Xinyu and Aparicio Ojea, Juan and Goldberg, Ken},
  title     = {Dex-{Net} 2.0: Deep Learning to Plan Robust Grasps with Synthetic Point Clouds and Analytic Grasp Metrics},
  booktitle = {Robotics: Science and Systems XIII},
  year      = {2017},
  doi       = {10.15607/RSS.2017.XIII.058}
}

@inproceedings{morrison2018ggcnn,
  author    = {Morrison, Douglas and Corke, Peter and Leitner, Jurgen},
  title     = {Closing the Loop for Robotic Grasping: A Real-time, Generative Grasp Synthesis Approach},
  booktitle = {Robotics: Science and Systems XIV},
  year      = {2018},
  doi       = {10.15607/RSS.2018.XIV.021}
}

@article{tenpas2017gpd,
  author    = {ten Pas, Andreas and Gualtieri, Marcus and Saenko, Kate and Platt, Robert},
  title     = {Grasp Pose Detection in Point Clouds},
  journal   = {The International Journal of Robotics Research},
  volume    = {36},
  number    = {13--14},
  pages     = {1455--1473},
  year      = {2017},
  doi       = {10.1177/0278364917735594}
}

@inproceedings{sundermeyer2021contact,
  author    = {Sundermeyer, Martin and Mousavian, Arsalan and Triebel, Rudolph and Fox, Dieter},
  title     = {Contact-{GraspNet}: Efficient {6-DoF} Grasp Generation in Cluttered Scenes},
  booktitle = {Proc. IEEE International Conference on Robotics and Automation (ICRA)},
  pages     = {13438--13444},
  year      = {2021},
  doi       = {10.1109/ICRA48506.2021.9561877}
}

@article{fang2023anygrasp,
  author    = {Fang, Hao-Shu and Wang, Chenxi and Fang, Hongjie and Gou, Minghao and Liu, Jirong and Yan, Hengxu and Liu, Wenhai and Xie, Yichen and Lu, Cewu},
  title     = {{AnyGrasp}: Robust and Efficient Grasp Perception in Spatial and Temporal Domains},
  journal   = {IEEE Transactions on Robotics},
  volume    = {39},
  number    = {5},
  pages     = {3929--3945},
  year      = {2023},
  doi       = {10.1109/TRO.2023.3281153}
}

@article{newbury2023review,
  author    = {Newbury, Rhys and Gu, Morris and Chumbley, Lachlan and Mousavian, Arsalan and Eppner, Clemens and Leitner, Jurgen and Bohg, Jeannette and Morales, Antonio and Asfour, Tamim and Kragic, Danica and Fox, Dieter and Cosgun, Akansel},
  title     = {Deep Learning Approaches to Grasp Synthesis: A Review},
  journal   = {IEEE Transactions on Robotics},
  volume    = {39},
  number    = {5},
  pages     = {3994--4015},
  year      = {2023},
  doi       = {10.1109/TRO.2023.3280597}
}

@book{vovk2005,
  author    = {Vovk, Vladimir and Gammerman, Alex and Shafer, Glenn},
  title     = {Algorithmic Learning in a Random World},
  publisher = {Springer},
  year      = {2005},
  doi       = {10.1007/b106715}
}

@article{angelopoulos2023,
  author    = {Angelopoulos, Anastasios N. and Bates, Stephen},
  title     = {Conformal Prediction: A Gentle Introduction},
  journal   = {Foundations and Trends in Machine Learning},
  volume    = {16},
  number    = {4},
  pages     = {494--591},
  year      = {2023},
  doi       = {10.1561/2200000101}
}

@inproceedings{lakshminarayanan2017,
  author    = {Lakshminarayanan, Balaji and Pritzel, Alexander and Blundell, Charles},
  title     = {Simple and Scalable Predictive Uncertainty Estimation Using Deep Ensembles},
  booktitle = {Advances in Neural Information Processing Systems 30 (NeurIPS)},
  pages     = {6402--6413},
  year      = {2017},
  eprint    = {1612.01474}
}

@inproceedings{todorov2012mujoco,
  author    = {Todorov, Emanuel and Erez, Tom and Tassa, Yuval},
  title     = {{MuJoCo}: A Physics Engine for Model-Based Control},
  booktitle = {Proc. IEEE/RSJ International Conference on Intelligent Robots and Systems (IROS)},
  pages     = {5026--5033},
  year      = {2012},
  doi       = {10.1109/IROS.2012.6386109}
}

@article{wei2022coacd,
  author    = {Wei, Xinyue and Liu, Minghua and Ling, Zhan and Su, Hao},
  title     = {Approximate Convex Decomposition for {3D} Meshes with Collision-Aware Concavity and Tree Search},
  journal   = {ACM Transactions on Graphics},
  volume    = {41},
  number    = {4},
  pages     = {42:1--42:18},
  year      = {2022},
  doi       = {10.1145/3528223.3530103}
}

@inproceedings{he2016resnet,
  author    = {He, Kaiming and Zhang, Xiangyu and Ren, Shaoqing and Sun, Jian},
  title     = {Deep Residual Learning for Image Recognition},
  booktitle = {Proc. IEEE Conference on Computer Vision and Pattern Recognition (CVPR)},
  pages     = {770--778},
  year      = {2016},
  doi       = {10.1109/CVPR.2016.90}
}

@inproceedings{zhou2019rotation,
  author    = {Zhou, Yi and Barnes, Connelly and Lu, Jingwan and Yang, Jimei and Li, Hao},
  title     = {On the Continuity of Rotation Representations in Neural Networks},
  booktitle = {Proc. IEEE/CVF Conference on Computer Vision and Pattern Recognition (CVPR)},
  pages     = {5738--5746},
  year      = {2019},
  doi       = {10.1109/CVPR.2019.00589}
}

@inproceedings{loshchilov2019adamw,
  author    = {Loshchilov, Ilya and Hutter, Frank},
  title     = {Decoupled Weight Decay Regularization},
  booktitle = {Proc. International Conference on Learning Representations (ICLR)},
  year      = {2019},
  eprint    = {1711.05101}
}

@inproceedings{deng2009imagenet,
  author    = {Deng, Jia and Dong, Wei and Socher, Richard and Li, Li-Jia and Li, Kai and Fei-Fei, Li},
  title     = {{ImageNet}: A Large-Scale Hierarchical Image Database},
  booktitle = {Proc. IEEE Conference on Computer Vision and Pattern Recognition (CVPR)},
  pages     = {248--255},
  year      = {2009},
  doi       = {10.1109/CVPR.2009.5206848}
}

@inproceedings{weisz2012poseerror,
  author    = {Weisz, Jonathan and Allen, Peter K.},
  title     = {Pose Error Robust Grasping from Contact Wrench Space Metrics},
  booktitle = {Proc. IEEE International Conference on Robotics and Automation (ICRA)},
  pages     = {557--562},
  year      = {2012},
  doi       = {10.1109/ICRA.2012.6224697}
}

@inproceedings{mahler2016dexnet1,
  author    = {Mahler, Jeffrey and Pokorny, Florian T. and Hou, Brian and Roderick, Melrose and Laskey, Michael and Aubry, Mathieu and Kohlhoff, Kai and Kr{\"o}ger, Torsten and Kuffner, James and Goldberg, Ken},
  title     = {{Dex-Net 1.0}: A Cloud-Based Network of {3D} Objects for Robust Grasp Planning Using a Multi-Armed Bandit Model with Correlated Rewards},
  booktitle = {Proc. IEEE International Conference on Robotics and Automation (ICRA)},
  pages     = {1957--1964},
  year      = {2016},
  doi       = {10.1109/ICRA.2016.7487342}
}

@article{hsiao2011robust,
  author    = {Hsiao, Kaijen and Kaelbling, Leslie Pack and Lozano-P{\'e}rez, Tom{\'a}s},
  title     = {Robust Grasping under Object Pose Uncertainty},
  journal   = {Autonomous Robots},
  volume    = {31},
  number    = {2--3},
  pages     = {253--268},
  year      = {2011},
  doi       = {10.1007/s10514-011-9243-2}
}

@article{lall1996nnboot,
  author    = {Lall, Upmanu and Sharma, Ashish},
  title     = {A Nearest Neighbor Bootstrap for Resampling Hydrologic Time Series},
  journal   = {Water Resources Research},
  volume    = {32},
  number    = {3},
  pages     = {679--693},
  year      = {1996},
  doi       = {10.1029/95WR02966}
}

@article{efron1979bootstrap,
  author    = {Efron, Bradley},
  title     = {Bootstrap Methods: Another Look at the Jackknife},
  journal   = {The Annals of Statistics},
  volume    = {7},
  number    = {1},
  pages     = {1--26},
  year      = {1979},
  doi       = {10.1214/aos/1176344552}
}

@inproceedings{brachmann2014lmo,
  author    = {Brachmann, Eric and Krull, Alexander and Michel, Frank and Gumhold, Stefan and Shotton, Jamie and Rother, Carsten},
  title     = {Learning {6D} Object Pose Estimation Using {3D} Object Coordinates},
  booktitle = {Proc. European Conference on Computer Vision (ECCV)},
  series    = {Lecture Notes in Computer Science},
  pages     = {536--551},
  year      = {2014},
  doi       = {10.1007/978-3-319-10605-2_35}
}

@article{deng2021poserbpf,
  author    = {Deng, Xinke and Mousavian, Arsalan and Xiang, Yu and Xia, Fei and Bretl, Timothy and Fox, Dieter},
  title     = {{PoseRBPF}: A {Rao--Blackwellized} Particle Filter for {6-D} Object Pose Tracking},
  journal   = {IEEE Transactions on Robotics},
  volume    = {37},
  number    = {5},
  pages     = {1328--1342},
  year      = {2021},
  doi       = {10.1109/TRO.2021.3056043}
}

@article{shi2021fastuq,
  author    = {Shi, Guanya and Zhu, Yifeng and Tremblay, Jonathan and Birchfield, Stan and Ramos, Fabio and Anandkumar, Animashree and Zhu, Yuke},
  title     = {Fast Uncertainty Quantification for Deep Object Pose Estimation},
  journal   = {arXiv preprint arXiv:2011.07748},
  year      = {2020},
  eprint    = {2011.07748}
}

@article{murphy2021ipdf,
  author    = {Murphy, Kieran and Esteves, Carlos and Jampani, Varun and Ramalingam, Srikumar and Makadia, Ameesh},
  title     = {{Implicit-PDF}: Non-Parametric Representation of Probability Distributions on the Rotation Manifold},
  journal   = {arXiv preprint arXiv:2106.05965},
  year      = {2021},
  eprint    = {2106.05965}
}

@article{haugaard2023spyropose,
  author    = {Haugaard, Rasmus Laurvig and Hagelskj{\ae}r, Frederik and Iversen, Thorbj{\o}rn Mosekj{\ae}r},
  title     = {{SpyroPose}: {SE(3)} Pyramids for Object Pose Distribution Estimation},
  journal   = {arXiv preprint arXiv:2303.05308},
  year      = {2023},
  eprint    = {2303.05308}
}

\end{document}